\documentclass{article}

\usepackage{arxiv}
\usepackage[T1]{fontenc}
\usepackage[utf8]{inputenc} % allow utf-8 input
\usepackage{hyperref}       % hyperlinks
\usepackage{url}            % simple URL typesetting
\usepackage{booktabs}       % professional-quality tables
\usepackage{amsfonts}       % blackboard math symbols
\usepackage{nicefrac}       % compact symbols for 1/2, 
\usepackage{microtype}      % microtypography
\usepackage{amsmath}
\usepackage{amsthm}
\usepackage{natbib}
\usepackage[final]{pdfpages}
\usepackage{adjustbox}
\usepackage{caption}
\usepackage{subfig}
\usepackage{tikz-cd}
\usepackage{caption}
\DeclareMathOperator{\pdim}{Pdim}

\newtheorem{lemma}{\textbf{Lemma}}
\newtheorem{theorem}{\textbf{Theorem}}
\newtheorem{corollary}{\textbf{Corollary}}
\DeclareMathOperator{\rad}{Rad}

\title{A Distribution Dependent and Independent Complexity Analysis of Manifold Regularization}

\author{
   Alexander Mey \\
  Delft University of Technology, The Netherlands\\
  \texttt{a.mey@tudelft.nl} \\
   \And
   Tom Viering \\
  Delft University of Technology, The Netherlands\\
  The Netherlands \\
  \texttt{t.j.viering@tudelft.nl} \\
   \AND
   Marco Loog\\
  Delft University of Technology, The Netherlands\\
University of Copenhagen, Denmark \\
  \texttt{m.loog@tudelft.nl} \\
}

\begin{document}
\maketitle
\begin{abstract}
Manifold regularization is a commonly used technique in semi-supervised learning. It enforces the classification rule to be smooth with respect to the data-manifold. Here, we derive sample complexity bounds based on pseudo-dimension for models that add a convex data dependent regularization term to a supervised learning process, as is in particular done in Manifold regularization. We then compare the bound for those semi-supervised methods to purely supervised methods, and discuss a setting in which the semi-supervised method can only have a constant improvement, ignoring logarithmic terms. By viewing Manifold regularization as a kernel method we then derive Rademacher bounds which allow for a distribution \emph{dependent} analysis. Finally we illustrate that these bounds may be useful for choosing an appropriate manifold regularization parameter in situations with very sparsely labeled data. 

\keywords{Semi-Supervised Learning  \and Learning Theory \and Manifold Regularization.}
\end{abstract}

\section{Introduction}
In many applications, as for example image or text classification, gathering unlabeled data is easier than gathering labeled data. Semi-supervised methods try to extract information from the unlabeled data to get improved classification results over purely supervised methods. A well-known technique to incorporate unlabeled data into a learning process is manifold regularization (MR) \citep{Belkin2,Manifold2}. This procedure adds a data-dependent penalty term to the loss function that penalizes classification rules that behave non-smooth with respect to the data distribution. This paper presents a sample complexity and a Rademacher complexity analysis for this procedure. In addition it illustrates how our Rademacher complexity bounds may be used for choosing a suitable Manifold regularization parameter. 

We organize this paper as follows. In Sections \ref{relatedwork} and \ref{SS} we discuss related work and introduce the semi-supervised setting. In Section \ref{PF} we formalize the idea of adding a distribution-dependent penalty term to a loss function. Algorithms such as manifold, entropy or co-regularization \citep{Belkin2,Entropy,Sindhwani2} follow this idea. Our formalization of this idea is inspired by \cite{Balcan1} and allows for a similar sample complexity analysis.  Section \ref{LB} reviews the work from Balcan et al. \cite{Balcan1} and generalizes a bound from their paper. We use this to derive sample complexity bounds for the proposed framework, and thus in particular for MR. For the specific case of regression, we furthermore adapt a sample complexity bound from \cite{Anthony}, which is essentially tighter than the first bound, to the semi-supervised case.  In the same section we sketch a setting in which we show that if our hypothesis set has finite pseudo-dimension, and we ignore logarithmic factors, any semi-supervised learner (SSL) that falls in our framework has at most a constant improvement in terms of sample complexity. This and related behavior has been observed and investigated before by \cite{Darnstadt} and \cite{BenDavid} for assumption free SSL and we relate our results to this previous work.
In Section \ref{Rademacher} we show how one can obtain distribution \emph{dependent} complexity bounds for MR. We review a kernel formulation of MR \citep{Sindhwani1} and show how this can be used to estimate Rademacher complexities for \emph{specific} datasets.  In Section \ref{experiments} we illustrate on an artificial dataset how the distribution dependent bounds could be used for choosing the regularization parameter of MR. This is particularly useful as the analysis does not need an additional labeled validation set. The practicality of this approach requires further empirical investigation.
In Section \ref{D} we discuss our results and speculate about possible extensions.

\section{Related Work} \label{relatedwork}
There are currently two related analyses of MR that show, to some extent, that a SSL can learn efficiently if it knows the true underlying manifold, while a fully supervised learner may not. \cite{Globerson} investigates a setting where distributions on the input space $\mathcal{X}$ are restricted to ones that correspond to unions of irreducible algebraic sets of a fixed size $k \in \mathbb{N}$, and each algebraic set is either labeled $0$ or $1$. A SSL that knows the true distribution on $\mathcal{X}$ can identify the algebraic sets and reduce the hypothesis space to all $2^k$ possible label combinations on those sets. As we are left with finitely many hypotheses we can learn them efficiently, while they show that every supervised learner is left with a hypothesis space of infinite VC dimension. 

The work of \cite{Manifold2} considers manifolds that arise as embeddings from a circle, where the labeling over the circle is (up to the decision boundary) smooth. They then show that a learner that has knowledge of the manifold can learn efficiently while for every fully supervised learner one can find an embedding and a distribution for which this is not possible.

 The relation to our paper is as follows. They provide specific examples where the sample complexity between a semi-supervised and a supervised learner are infinitely large, while we explore general sample complexity bounds of MR and  sketch a setting in which MR can not essentially improve over supervised methods.
 
\section{The Semi Supervised Setting} \label{SS}
We work in the statistical learning framework: we assume we are given a feature domain $\mathcal{X}$ and an output space $\mathcal{Y}$ together with an unknown probability distribution $P$ over $ \mathcal{X} \times \mathcal{Y}$. In binary classification we usually have that $\mathcal{Y}=\{-1,1\}$, while for regression $\mathcal{Y}=\mathbb{R}$. We use a loss function $\phi: \mathbb{R} \times \mathcal{Y} \to \mathbb{R}$, which is convex in the first argument and in practice usually a surrogate for the $0$-$1$ loss in classification, and the squared loss in regression tasks. A hypothesis $f$ is a function $f:\mathcal{X} \to \mathbb{R}$. We set $(X,Y)$ to be a random variable distributed according to $P$, while small $x$ and $y$ are elements of $\mathcal{X}$ and $\mathcal{Y}$ respectively. Our goal is to find a hypothesis $f$, within a restricted class $\mathcal{F}$, such that the expected loss $Q(f):=\mathbb{E}[\phi(f(X),Y)]$ is small. In the standard supervised setting we choose a hypothesis $f$ based on an i.i.d. sample $S_n=\{(x_i,y_i)\}_{i \in \{ 1,..,n\}}$ drawn from $P$. With that we define the empirical risk of a model $f \in \mathcal{F}$ with respect to $\phi$ and measured on the sample $S_n$ as $\hat{Q}(f,S_n)= \frac{1}{n}\sum_{i =1}^{n} \phi(f(x_i),y_i).$ For ease of notation we sometimes omit $S_n$ and just write $\hat{Q}(f)$. Given a learning problem defined by $(P,\mathcal{F},\phi)$ and a labeled sample $S_n$, one way to choose a hypothesis is by the empirical risk minimization principle
\begin{equation} \label{ERM}
f_{\text{sup}} = \arg \min \limits_{f \in \mathcal{F}} \hat{Q}(f,S_n).
\end{equation}
We refer to $f_{\text{sup}}$ as the \emph{supervised solution}. In SSL we additionally have samples with unknown labels. So we assume to have $n+m$ samples $(x_i,y_i)_{i \in \{ 1,..,n+m\}}$ independently drawn according to $P$, where  $y_i$ has not been observed for the last $m$ samples. We furthermore set $U=\{x_1,...,x_{x_n+m} \}$, so $U$ is the set that contains all our available information about the feature distribution.

Finally we denote by $m^{L}(\epsilon,\delta)$ the sample complexity of an algorithm $L$. That means that for all $n\geq m^L(\epsilon,\delta)$ and all possible distributions $P$ the following holds. If $L$ outputs a hypothesis $f_L$ after seeing an $n$-sample, we have with probability of at least $1-\delta$ over the $n$-sample $S_n$ that
$
Q(f_L) - \min\limits_{f \in \mathcal{F}}Q(f) \leq \epsilon.
$

\section{A Framework for Semi-Supervised Learning} \label{PF}

We follow the work of \cite{Balcan1} and introduce a second convex loss function $\psi: \mathcal{F} \times \mathcal{X} \to \mathbb{R}_+$ that only depends on the input feature and a hypothesis. We refer to $\psi$ as the \emph{unsupervised loss} as it does not depend on any labels. We propose to \emph{add} the unlabeled data through the loss function $\psi$ and add it as a penalty term to the supervised loss to obtain the semi-supervised solution
\begin{equation} \label{ERMSSL}
f_{\text{semi}} =\arg \min_{f \in \mathcal{F}} \frac{1}{n} \sum_{i=1}^n \phi(f(x_i),y_i) +\lambda \frac{1}{n+m} \sum_{j=1}^{n+m} \psi(f,x_j),
\end{equation}
where $\lambda >0$ controls the trade-off between the supervised and the unsupervised loss. This is in contrast to \cite{Balcan1}, as they use the unsupervised loss to restrict the hypothesis space directly. In the following section we recall the important insight that those two formulations are equivalent in some scenarios and we can use \citep{Balcan1} to generate sample complexity bounds for the here presented SSL framework.

For ease of notation we set $\hat{R}(f,U)=\frac{1}{n+m}\sum_{j=1}^{n+m} \psi(f,x_j)$ and $R(f)=\mathbb{E}[\psi(f,X)].$
We do not claim any novelty for the idea of adding an unsupervised loss for regularization. A different framework can be found in \cite[Chapter 10]{Chapelle}. We are, however, not aware of a deeper analysis of this particular formulation, as done for example by the sample complexity analysis in this paper. As we are in particular interested in the class of MR schemes we first show that this method a fits our framework.
\paragraph{Example: Manifold Regularization}
Overloading the notation we write now $P(X)$ for the distribution $P$ restricted to $\mathcal{X}$. In MR one assumes that the input distribution $P(X)$ has support on a compact manifold ${M} \subset \mathcal{X}$ and that the predictor $f \in \mathcal{F}$ varies smoothly in the geometry of ${M}$ \citep{Belkin2}. There are several regularization terms that can enforce this smoothness, one of which is $ \int_{M} || \nabla_{M} f(x) ||^2 dP(x)$, where $\nabla_{M} f$ is the gradient of $f$ along $M$. We know that $ \int_{M} || \nabla_M f(x) ||^2 dP(x)$ may be approximated with a finite sample of $\mathcal{X}$ drawn from $P(X)$ \citep{Belkin3}. Given such a sample $U=\{x_1,...,x_{n+m}\}$ one defines first a weight matrix $W$, where $W_{ij}=e^{-{||x_i-x_j ||^2/}{\sigma}}$. We set $L$ then as the Laplacian matrix $L=D-W$, where $D$ is a diagonal matrix with $D_{ii}=\sum_{j=1}^{n+m}W_{ij}$. Let furthermore $f_U=(f(x_1),...,f(x_{n+m}))^t$ be the evaluation vector of $f$ on $U$. The expression $\frac{1}{(n+m)^2}f_U^tLf_U=\frac{1}{(n+m)^2} \sum_{i,j}(f(x_i)-f(x_j))^2W_{ij}$  converges to $ \int_M || \nabla_M f ||^2 dP(x)$ under certain conditions \citep{Belkin3}. This motivates us to set the unsupervised loss as
$\psi(f,(x_i,x_j))=(f(x_i)-f(x_j))^2W_{ij}$, and this is indeed a convex function in $f$.

\section{Analysis of the Framework} \label{LB}

In this section we analyze the properties of the solution $f_{\text{semi}}$ found in Equation \eqref{ERMSSL}. We derive sample complexity bounds for this procedure, using results from \cite{Balcan1}, and compare them to sample complexities for the supervised case. In \cite{Balcan1} the unsupervised loss is used to restrict the hypothesis space directly, while we use it as a regularization term in the empirical risk minimization as usually done in practice. To switch between the views of a constrained optimization formulation and our formulation \eqref{ERMSSL} we use the following classical result from convex optimization \citep[Theorem 1]{Kloft}.

\begin{lemma} \label{Lemma1}
Let $\phi(f(x),y)$ and $\psi(f,x)$ be functions convex in $f$ for all $x,y$. Then the following two optimization problems are equivalent:
\begin{equation} \label{regularized}
\min_{f \in \mathcal{F}} \frac{1}{n}\sum_{i=1}^{n}\phi(f(x_i),y_i) + \lambda \frac{1}{n+m}\sum_{i=1}^{n+m} \psi(f,x_i)
\end{equation}
\begin{equation} \label{constrained}
\min_{f \in \mathcal{F}} \frac{1}{n}\sum_{i=1}^{n}\phi(f(x_i),y_i) \text{~~subject to~} \sum_{i=1}^{n+m} \frac{1}{n+m} \psi(f,x_i) \leq \tau
\end{equation}
Where equivalence means that for each $\lambda$ we can find a $\tau$ such that both problems have the same solution and vice versa.
\end{lemma}
For our later results we will need the conditions of this lemma are true, which we believe to be not a strong restriction. In our sample complexity analysis we stick as close as possible to the actual formulation and implementation of MR, which is usually a convex optimization problem.
We now first turn to our sample complexity bounds.

%The next subsection introduces the sample complexity bound and shows how it can be used to give theoretical guarantees for the presented framework.

\subsection{Sample Complexity Bounds}
Sample complexity bounds for supervised learning use typically a notion of complexity of the hypothesis space to bound the worst case difference between the estimated and the true risk. % One well-known capacity notion for classification functions is the VC-dimension \cite{Vapnik}.
As our hypothesis class allows for real-valued functions, we will use the notion of pseudo-dimension $\pdim(\mathcal{F},\phi)$, an extension of the VC-dimension to real valued loss functions $\phi$ and hypotheses classes $\mathcal{F}$ \citep{Vapnik,foundation}. Informally speaking, the pseudo-dimension is the VC-dimension of the set of functions that arise when we threshold real-valued functions to define binary functions. Note that sometimes the pseudo-dimension will have as input the loss function, and sometimes not. This is because some results use the concatenation of loss function and hypotheses to determine the capacity, while others only use the hypotheses class. This lets us state our first main result, which is a generalization of \cite[Theorem~10]{Balcan1} to bounded loss functions and real valued function spaces.

\begin{theorem} \label{BoundThm}
Let $\mathcal{F}^{\psi}_{\tau}:=\{f \in \mathcal{F} \ | \ \mathbb{E}[\psi(f,x)] \leq \tau \}$.
Assume that $\phi, \psi$ are measurable loss functions such that there exists constants $B_1,B_2>0$  with $\psi(f,x) \leq B_1$ and $\phi(f(x),y) \leq B_2$ for all $x,y$ and $f \in \mathcal{F}$ and let $P$ be a distribution. Furthermore let $f^*_{\tau}=\arg \min \limits_{f \in \mathcal{F}^{\psi}_{\tau}} Q(f)$.
Then an unlabeled sample $U$ of size  $$m \geq \frac{8{B_1}^2}{\epsilon^2}\left[ \ln\frac{16}{\delta}+2\pdim(\mathcal{F},\psi)\ln\frac{4B_1}{\epsilon} +1 \right] $$ and a labeled sample $S_n$ of size $$n\geq\max\left( \frac{8{B_2}^2}{\epsilon^2}\left[\ln\frac{8}{\delta}+2\pdim(\mathcal{F}^{\psi}_{\tau+\frac{\epsilon}{2}},\phi)\ln\frac{4B_2}{\epsilon}+1 \right],\frac{h}{4}\right)$$  is sufficient to ensure that with probability at least $1-\delta$ the classifier $g \in \mathcal{F}$ that minimizes $\hat{Q}(\cdot,S_n)$ subject to $\hat{R}(\cdot,U) \leq \tau + \frac{\epsilon}{2}$ satisfies
\begin{equation}\label{bound}
Q(g) \leq Q(f_{\tau}^*) + \epsilon .
\end{equation}

\end{theorem}
\begin{proof}
First we show that the unlabeled sample size is big enough to guarantee that with probability at least $1-\frac{\delta}{4}$ it holds that  $\hat{R}(f^*_{\tau}) \leq \tau +\frac{\epsilon}{2}$. For $h=\pdim(\mathcal{F},\psi)$ Theorem 5.1 from \cite{Vapnik} states that
$$
P\left[ \underset{f \in \mathcal{F}}{\sup} (\hat{R}(f)-R(f)) > \frac{\epsilon}{2} \right]
\leq 4 e^{ h(\ln\frac{2m}{h}+1) -\frac{m}{{B_1}^2}(\frac{\epsilon}{2}-\frac{1}{m})^2 }.
$$
Bounding 
$$
 4 e^{h(\ln\frac{2m}{h}+1)-\frac{m}{{B_1}^2}(\frac{\epsilon}{2}-\frac{1}{m})^2 } \leq \frac{\delta}{4}
$$
and rewriting this gives us that
$$
m \geq \frac{4{B_1}^2}{\epsilon^2} \left[\ln\frac{16}{\delta}+h \ln\frac{2em}{h} \right]
=\frac{4{B_1}^2}{\epsilon^2} \left[\ln\frac{16}{\delta}+h \ln m +h \ln\frac{2e}{h}+1 \right]
$$
is sufficient to ensure that $\hat{R}(f)-R(f) < \frac{\epsilon}{2}$ for all $f \in \mathcal{F}$ with probability at least $1-\frac{\delta}{4}$. Using the inequality $\ln x \leq \alpha x -\ln \alpha -1$ with $x=m$ and $\alpha=\frac{\epsilon^2}{8h{B_1}^2}$ we can conclude that a sample of size
\begin{equation*}
\begin{split}
 m & \geq \frac{4{B_1}^2}{\epsilon^2} \left[ \ln\frac{16}{\delta}+h (\frac{\epsilon^2}{8h{B_1}^2}m+\ln \frac{8h{B_1}^2}{\epsilon^2}-1) +h \ln\frac{2e}{h}+1 \right] \\
 & =\frac{m}{2}+\frac{4{B_1}^2}{\epsilon^2}\left[ \ln\frac{16}{\delta}+h\ln\frac{16{B_1}^2}{\epsilon^2}+1 \right] \\
& \iff \\
 m &\geq \frac{8{B_1}^2}{\epsilon^2}\left[\ln\frac{16}{\delta}+2h\ln\frac{4{B_1}}{\epsilon}+1 \right]
\end{split}
\end{equation*}
is sufficient to guarantee $\hat{R}(f)-R(f) < \frac{\epsilon}{2}$ for all $f \in \mathcal{F}$ with probability at least $1-\frac{\delta}{4}$. In particular choosing $f=f^*_\tau$ and noting that by definition $R(f^*_\tau) \leq \tau$ we conclude that with the same probability  \begin{equation} \label{Firsteq}
\hat{R}(f^*_{\tau}) \leq \tau +\frac{\epsilon}{2}.
\end{equation}

For the second part we use the classical Hoeffding inequality with a labeled sample size of $n$
$$
P\left[\hat{Q}(f^*_{\tau})-Q(f^*_{\tau}) \geq \theta \right] \leq e^{\frac{ -2 \theta^2 n}{{B_2}^2}} .
$$
Choosing $\theta =B_2\sqrt{\ln (\frac{4}{\delta}  )\frac{1}{2n}} $
lets us conclude that with probability at least $1-\frac{\delta}{4}$ it holds that 
\begin{equation}\label{Second}
\hat{Q}(f^*_{\tau}) \leq Q(f^*_{\tau}) + B_2\sqrt{\ln (\frac{4}{\delta}  )\frac{1}{2n}}.\end{equation} \\
For the third part we use again Theorem 5.1 from \cite{Vapnik} with $h=\pdim(\mathcal{F}^{\psi}_{\tau},\phi)$, which states that
\begin{equation} \label{vapniklabeled}
n \geq \frac{4{B_2}^2}{\epsilon^2} \left[\ln\frac{8}{\delta}+h \ln\frac{2en}{h}+1 \right]
\end{equation}
is sufficient to guarantee with probability at least $1-\frac{\delta}{2}$ that 
 \begin{equation} \label{Third}
 Q(f) -\hat{Q}(f) \leq \frac{\epsilon}{2} \text{ \ for all \ }f \in \mathcal{F}^{\psi}_{\tau+\frac{\epsilon}{2}}.
 \end{equation} 
With the same reasoning as for the first part we obtain the same guarantee with a labeled sample of size
$$n\geq \frac{8{B_2}^2}{\epsilon^2}\left[\ln\frac{8}{\delta}+2h\ln\frac{4{B_2}}{\epsilon}+1 \right].
$$
 Putting everything together with we get, using the union bound, that with probability $1-\delta$  the classifier $g$ that minimizes $\hat{Q}(\cdot,X,Y)$ subject to $\hat{R}(\cdot,U) \leq \tau + \frac{\epsilon}{2}$ satisfies  $$Q(g) \leq \hat{Q}(g) +\frac{\epsilon}{2} \leq \hat{Q}(f_{\tau}^*) + \frac{\epsilon}{2} \leq Q(f_{\tau}^*)  +\frac{\epsilon}{2} + B_2\sqrt{\frac{\ln(\frac{4}{\delta})}{2n}} .$$
The first inequality follows from Inequality \eqref{Third}. The second inequality follows because $g$ is the empirical minimizer. Note that we also need Inequality \eqref{Firsteq}, i.e. that $\hat{R}(f^*_{\tau}) \leq \tau +\frac{\epsilon}{2}$, to make sure that $f^*_{\tau}$ was in the search space. The third inequality follows from Inequality \eqref{Second}. To obtain the final inequality we use the labeled sample size to show that
$$
\frac{\epsilon}{2} \geq \sqrt{ \frac{{B_2}^2}{n}\left[\ln \frac{8}{\delta}+h\ln\frac{2en}{h}+1\right]} \geq B_2\sqrt{\frac{\ln(\frac{4}{\delta})}{2n}}.$$
The first inequality holds by assumption of the labeled sample size from Inequality \eqref{vapniklabeled}, while the second inequality is shown by reducing it to 
$$
h\ln\frac{2en}{h}+1 \geq \frac{1}{2}\ln (\frac{1}{2})
$$ which holds as the right-hand side is negative, while the left-hand side is positive as $2en>h$ since by our assumptions $4n>h$. 
\end{proof}

The next subsection uses this theorem to derive sample complexity bounds for MR. First, however, a remark about the assumption that the loss function $\phi$ is globally bounded. If we assume that $\mathcal{F}$ is a reproducing kernel Hilbert space there exists an $M>0$ such that for all $f \in \mathcal{F}$ and $x \in \mathcal{X}$ it holds that $|f(x)| \leq M ||f||_\mathcal{F}$.
If we restrict the norm of $f$ by introducing a regularization term with respect to the norm $||.||_\mathcal{F}$, we know that the image of $\mathcal{F}$ is globally bounded. If the image is also closed it will be compact, and thus $\phi$ will be globally bounded in many cases, as most loss functions are continuous. This can also be seen as a justification to also use an intrinsic regularization for the norm of $f$ in addition to the regularization by the unsupervised loss, as only then the guarantees of Theorem \ref{BoundThm} apply.
Using this bound together with Lemma \ref{Lemma1} we can state the following corollary to give a PAC-style guarantee for our proposed framework.
\begin{corollary} \label{main_manifold}
Let $\phi$ and $\psi$ be convex supervised and an unsupervised loss function that fulfill the assumptions of Theorem \ref{BoundThm}. Then $f_\text{semi}$ \eqref{ERMSSL} satisfies the guarantees given in Theorem \ref{BoundThm}, when we replace for it $g$ in Inequality \eqref{bound}.
\end{corollary}
Recall that in the MR setting $\hat{R}(f)=\frac{1}{(n+m)^2}\sum_{i=1}^{n+m}W_{ij}(f(x_i)-f(x_j))^2$. So we gather unlabeled samples from $\mathcal{X} \times \mathcal{X}$ instead of $\mathcal{X}$. Collecting $m$ samples from $\mathcal{X}$ equates $m^2-1$ samples from $\mathcal{X} \times \mathcal{X}$ and thus we only need $\sqrt{m}$ instead of $m$ unlabeled samples for the same bound. 

\subsection{Comparison to the Supervised Solution}
In the SSL community it is well-known that using SSL does not come without a risk \citep[Chapter~4]{Chapelle}. Thus it is of particular interest how those methods compare to purely supervised schemes. There are, however, many potential supervised methods we can think of. In many works this problem is avoided by comparing to all possible supervised schemes \citep{BenDavid,Darnstadt,Globerson}. The framework introduced in this paper allows for a more fine-grained analysis as the semi-supervision happens on top of an already existing supervised methods. Thus, for our framework, it is natural to compare the sample complexities of $f_{\text{sup}}$ with the sample complexity of $f_{\text{semi}}$. To compare the supervised and semi-supervised solution we draw from \cite[Chapter 20]{Anthony}, where one can find lower and upper sample complexity bounds for the regression setting. To use this we have to restrict to the square loss, so in this section we set $\phi(f(x),y)=(f(x)-y)^2$. The main insight from \cite[Chapter 20]{Anthony} is that the sample complexity depends in this setting on whether the hypothesis class is (closure) convex or not. As we anyway need convexity of the space, which is stronger than closure convexity, to use Lemma \ref{Lemma1}, we can adapt Theorem 20.7 from \cite{Anthony} to our semi-supervised setting.

\begin{theorem} \label{UpperNN}
Assume that $\mathcal{F}^{\psi}_{\tau+\epsilon}$ is a closure convex class with functions mapping to $[0,1]$\footnote{In the remarks after Theorem \ref{BoundThm} we argue that in many cases |f(x)| is bounded, and in those cases we can always map to [0,1] by re-scaling.}, that $\psi(f,x)\leq B_1$ for all $x \in \mathcal{X}$ and $f \in \mathcal{F}$ and that $\phi(f(x),y)=(f(x)-y)^2$. Assume further that there is a $B_2>0$ such that $(f(x)-y)^2<B_2$ almost surely for all $(x,y) \in \mathcal{X} \times \mathcal{Y}$ and $f \in \mathcal{F}^{\psi}_{\tau+\epsilon}$.
Then an unlabeled sample size of $$m \geq \frac{2{B_1}^2}{\epsilon^2}\left[ \ln\frac{8}{\delta}+2\pdim(\mathcal{F},\psi)\ln\frac{2B_1}{\epsilon} +2 \right] $$
and a labeled sample size of 
\begin{equation} \label{complexNN}
n \geq \mathcal{O} \left( \frac{B^2}{\epsilon}\left(\pdim(\mathcal{F}^{\psi}_{\tau+\epsilon})\ln{\frac{\sqrt{B}}{\epsilon}}+\ln{\frac{2}{\delta}}\right) \right)
\end{equation}
is \emph{sufficient} to guarantee that with probability at least $1-\delta$ the classifier $g$ that minimizes $\hat{Q}( \cdot)$ w.r.t $\hat{R}(f) \leq \tau +\epsilon$ satisfies
\begin{equation}
    Q(g) \leq \min\limits_{f \in \mathcal{F^{\psi}_{\tau}}}Q(f)+\epsilon.
\end{equation}
\end{theorem}
\begin{proof} As in the proof of Theorem \ref{BoundThm} the unlabeled sample size is sufficient to guarantee with probability at least $1-\frac{\delta}{2}$ that ${R}(f^*_{\tau}) \leq \tau +\epsilon$. The labeled sample size is big enough to guarantee with at least $1-\frac{\delta}{2}$ that $Q(g) \leq Q(f^*_{\tau+\epsilon})+\epsilon$ \citep[Theorem 20.7]{Anthony}. Using the union bound we have with probability of at least $1-\delta$ that
$Q(g) \leq Q(f^*_{\tau+\epsilon})+\epsilon \leq Q(f^*_{\tau})+\epsilon$.
\end{proof}

Note that the previous theorem of course implies the same learning rate in the supervised case, as the only difference will be the pseudo-dimension term. As in specific scenarios this is also the best possible learning rate, we obtain the following negative result for SSL.

\begin{corollary} \label{corlimits}
Assume that $\mathcal{F}$ maps to the interval $[0,1]$ and $\mathcal{Y}=[1-B,B]$ for a $B \geq 2$. If $\mathcal{F}$ and $\mathcal{F}^{\psi}_{\tau}$ are both closure convex, then for sufficiently small $\epsilon,\delta>0$ it holds that $m^{\text{sup}}(\epsilon,\delta) = \tilde{\mathcal{O}}(m^{\text{semi}}(\epsilon,\delta))$, where $\tilde{\mathcal{O}}$ suppresses logarithmic factors, and $m^{\text{semi}},m^{\text{sup}}$ denote the sample complexity of the semi-supervised and the supervised learner respectively. In other words, the semi-supervised method can improve the learning rate by at most a constant which may depend on the pseudo-dimensions, ignoring logarithmic factors. Note that this holds in particular for the manifold regularization algorithm.
\end{corollary}
\begin{proof}
The assumptions made in the theorem allow is to invoke Equation (19.5) from \cite{Anthony} which states that $m^{\text{semi}}=\Omega(\frac{1}{\epsilon}+\pdim(\mathcal{F}^{\psi}_{\tau}))$.\footnote{Note that the original formulation is in terms of the fat-shattering dimension, but this is always bounded by the pseudo-dimension.}
Using Inequality \eqref{complexNN} as an upper bound for the supervised method and comparing this to Eq. (19.5) from \cite{Anthony} we observe that all differences are either constant or logarithmic in $\epsilon$ and $\delta$.
\end{proof} 

%\begin{proposition}
%Let $\mathcal{F}$ be a closure convex hypothesis class and $\psi(f,(u_i,u_j))=(f(u_i)-f(u_j))^2W_{ij}$, where $W_{ij}$ is defined as in the manifold regularization example. Then, given an unlabeled sample $U_m$, the hypothesis class $\mathcal{F}^{\psi}_{\tau}=\{f \in \mathcal{F} \mid \frac{1}{m^2} \sum_{u_i,u_j \in U_m} \psi(f,u_i,u_j) \leq \tau \}$ is also closure convex.
%\end{proposition}
%\begin{proof}
%First note that $\frac{1}{m^2} \sum_{u_i,u_j \in U_m} \psi(f,u_i,u_j)=\frac{1}{m^2} f_U^tLf_U$, where $L$ is the Laplacian matrix as before. Assume now that $f,g \in \mathcal{F}^{\psi}_{\tau}$, so that $\frac{1}{m^2} f_U^tLf_U \leq \tau$ and $\frac{1}{m^2} g_U^tLg_U \leq \tau$. For $0\leq \alpha \leq 1$ we have then that
%\begin{align*}
  %  & \frac{1}{m^2} (\alpha f +(1-\alpha)g)_U^tL(\alpha f +(1-\alpha)g)_U \\
 %   = & \frac{\alpha^2}{m^2} f_U^tLf_U+\frac{\alpha(1-\alpha)}{m^2} f_U^tLg_U+\frac{\alpha(1-\alpha)}{m^2} g_U^tLf_U+\frac{(1-\alpha)^2}{m^2} g_U^tLg_U \\
  %  \leq & \alpha^2 \tau +\alpha(1-\alpha) \tau+{\alpha(1-\alpha)}\tau+(1-\alpha)^2 \tau= \tau
%\end{align*}
%That means that $\alpha f + (1-\alpha)g \in \mathcal{F}^{\psi}_{\tau}$.
%The inequality follows with the assumptions that $f,g \in \mathcal{F}^{\psi}_{\tau}$ and the Cauchy-Schwarz inequality, noting that $L$ is a positive definite matrix and thus defines an inner product.
%\hfill $\square$
%\end{proof}
\subsection{The Limits of Manifold Regularization} \label{limits}
We now relate our result to the conjectures published in \cite{Shalev}: A SSL cannot learn faster by more than a constant (which may depend on the hypothesis class $\mathcal{F}$ and the loss $\phi$) than the supervised learner. Theorem 1 from \cite{Darnstadt} showed that this conjecture is true up to a logarithmic factor, much like our result, for classes with finite VC-dimension, and SSL that do \emph{not} make any distributional assumptions. Corollary \ref{corlimits} shows that this statement also holds in some scenarios for all SSL that fall in our proposed framework. This is somewhat surprising, as our result holds explicitly for SSLs that \emph{do} make assumptions about the distribution: MR assumes the labeling function behaves smoothly w.r.t. the underlying manifold.

\section{Rademacher Complexity of Manifold Regularization} \label{Rademacher}
%The analysis of the previous sections was distribution independent. We saw that the sample complexity difference between manifold egularized methods and purely supervised methods was at most a constant, which depends on the hypothesis class and the regularization parameter.
In order to find out in which scenarios semi-supervised learning can help it is useful to also look at distribution \emph{dependent} complexity measures. For this we derive computational feasible upper and lower bounds on the Rademacher complexity of MR. We first review the work of \cite{Sindhwani1}: they create a kernel such that the inner product in the corresponding kernel Hilbert space contains automatically the regularization term from MR. Having this kernel we can use standard upper and lower bounds of the Rademacher complexity for RKHS, as found for example in \cite{Boucheron}. The analysis is thus similar to \cite{Sindhwani2}. They consider a co-regularization setting. In particular \cite[p1]{Sindhwani1} show the following, here informally stated, theorem.
\begin{theorem}[{\cite[Propositions 2.1, 2.2]{Sindhwani1}}]
Let $H$ be a RKHS with inner product $\langle\cdot,\cdot\rangle_H$. As before let $U=\{x_1,...,x_{n+m}\}$, $f,g \in H$ and $f_U=(f(x_1),...,f(x_{n+m}))^t$. Furthermore let $\langle\cdot,\cdot\rangle_{\mathbb{R}^n}$ be any inner product in $\mathbb{R}^n$. Let $\tilde{H}$ be the same space of functions as $H$, but with a newly defined inner product by
$
\langle f,g\rangle_{\tilde{H}}=\langle f,g\rangle_H+\langle f_U,g_U\rangle_{\mathbb{R}^n}.
$
Then $\tilde{H}$ is a RKHS.
\end{theorem}
Assume now that $L$ is a positive definite $n$-dimensional matrix and we set the inner product $\langle f_U,g_U\rangle_{\mathbb{R}^n}=f_U^t L g_U.$ By setting $L$ as the Laplacian matrix (Section \ref{PF}) we note that the norm of $\tilde{H}$ automatically regularizes w.r.t. the data manifold given by $\{x_1,...,x_{n+m}\}$. We furthermore know the exact form of the kernel of $\tilde{H}$.
\begin{theorem}[{\cite[Proposition 2.2]{Sindhwani1}}]
Let $k(x,y)$ be the kernel of $H$, $K$ be the gram matrix given by $K_{ij}=k(x_i,x_j)$ and $k_x=(k(x_1,x),...,k(x_{n+m},x))^t$. Finally let $I$ be the $n+m$ dimensional identity matrix. The kernel of $\tilde{H}$ is then given by $\tilde{k}(x,y)=k(x,y)-k_x^t(I+LK)^{-1}Lk_y.$
\end{theorem}
This interpretation of MR is useful to derive computationally feasible upper and lower bounds of the empirical Rademacher complexity, giving distribution \emph{dependent} complexity bounds. With $\sigma=(\sigma_1,...,\sigma_n)$ i.i.d Rademacher random variables (i.e. $P(\sigma_i=1)=P(\sigma_i=-1)=\frac{1}{2}$.), recall that the empirical Rademacher complexity of the hypothesis class $H$ and measured on the sample labeled input features $\{x_1,...,x_n\}$ is defined as 
$$
\rad_n(H)=\frac{1}{n} \mathbb{E}_{\sigma} \sup\limits_{f \in H} \sum_{i=1}^n  \sigma_i f(x_i).
$$
\begin{theorem}[{\cite[p. 333]{Boucheron}}]
Let $H$ be a RKHS with kernel $k$ and $H_r=\{ f \in H \mid ||f||_H \leq r \}$. 
Given an $n$ sample $\{x_1,...,x_n\}$ we can bound the empirical Rademacher complexity of $H_r$ by
\begin{equation} \label{radbounds}
\frac{r}{n \sqrt{2}} \sqrt{\sum_{i=1}^n k(x_i,x_i)} \leq \rad_n(H_r) \leq \frac{r}{n} \sqrt{\sum_{i=1}^n k(x_i,x_i)}.
\end{equation}
\end{theorem}
The previous two theorems lead to upper bounds on the complexity of MR, in particular we can bound the maximal reduction over supervised learning.
\begin{corollary}
Let $H$ be a RKHS and for $f,g \in H$ define the inner product $\langle f,g\rangle_{\tilde{H}}=\langle f,g\rangle_{H}+f_U (\mu L) g_U^t$, where $L$ is a positive definite matrix and $\mu \in \mathbb{R}$ is a regularization parameter. Let $\tilde{H}_r$ be defined as before, then 
\begin{equation} \label{radeupper}
\rad_n(\tilde{H}_r) \leq \frac{r}{n} \sqrt{\sum_{i=1}^n k(x_i,x_i) -k^t_{x_i}(\frac{1}{\mu}I+ LK)^{-1} Lk_{x_i}}.
\end{equation}
Similarly we can obtain a lower bound in line with Inequality \eqref{radbounds}.
\end{corollary}
The corollary allows us to compute upper bounds of the Rademacher complexity for MR and
shows in particular that the difference of the Rademacher complexity of the supervised and the semi-supervised method is given by the term $k^t_{x_i}(\frac{1}{\mu}I_{n+m}+ LK)^{-1} Lk_{x_i}$. This can be used for example to compute generalization bounds \citep[Chapter 3]{foundation}. We can also use the kernel to compute local Rademacher complexities which may yield tighter generalization bounds \citep{bartlett2005}. Here we illustrate the use of our bounds for choosing the regularization parameter $\mu$ without the need for an additional labeled validation set.
\section{Experiment: Concentric circles} \label{experiments}
\begin{figure}
  \centering
\includegraphics[scale=0.6]{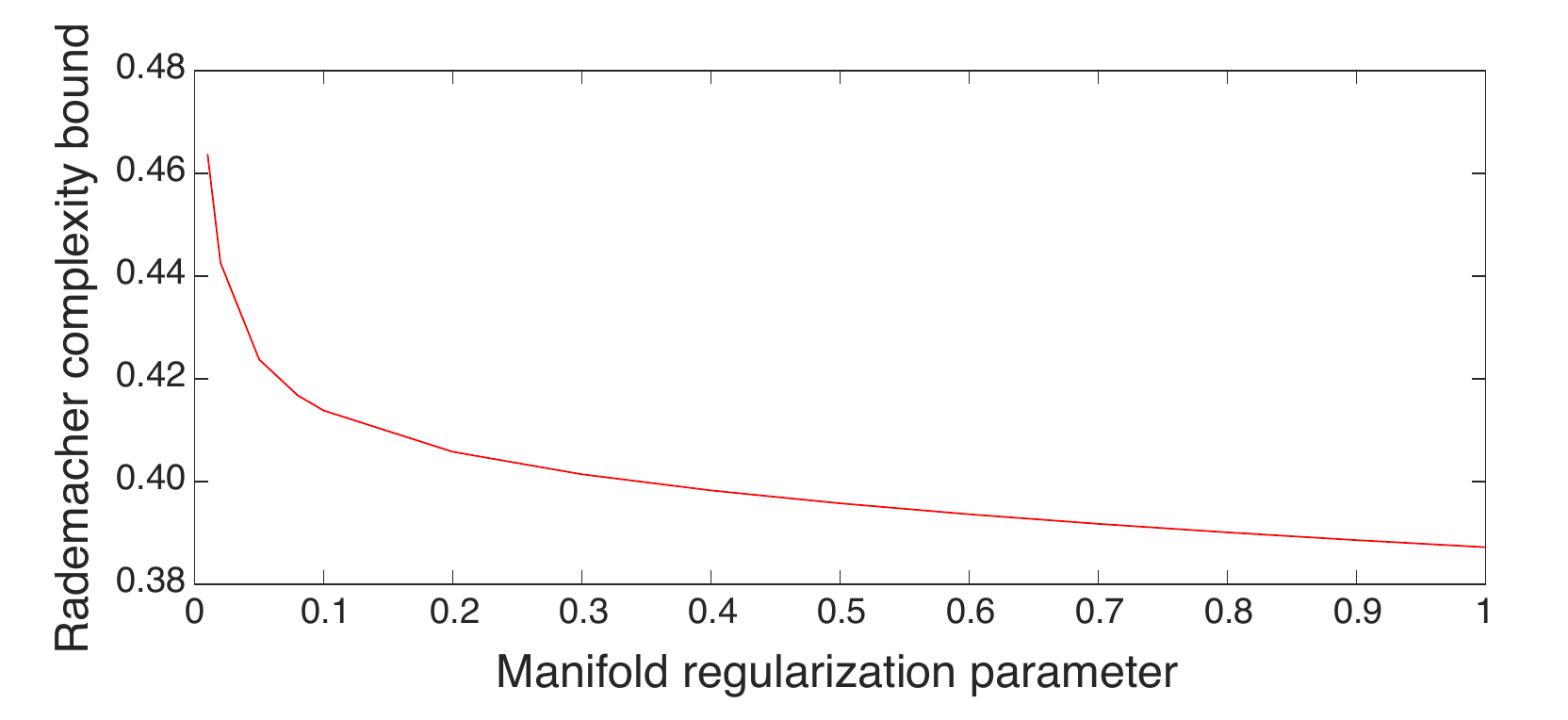}
  \caption{ The behavior of the Rademacher complexity when using manifold regularization on circle dataset with different regularization values $\mu$. }
  \label{Rademachercurve}
\end{figure}
We illustrate the use of Eq. \eqref{radeupper} for model selection. In particular, it can be used to get an initial idea of how to choose the regularization parameter $\mu$. The idea is to plot the Rademacher complexity versus the parameter $\mu$ as in Figure \ref{Rademachercurve}. We propose to use an heuristic which is often used in clustering, the so called elbow criteria \citep{Purnima}. We essentially want to find a $\mu$ such that increasing the $\mu$ will not result in much reduction of the complexity anymore. We test this idea on a dataset which consists out of two concentric circles with 500 datapoints in $\mathbb{R}^2$, 250 per circle, see also Figure \ref{circles_manifold}. We use a Gaussian base kernel with bandwidth set to $0.5$. The MR matrix $L$ is the Laplacian matrix, where weights are computed with a Gaussian kernel with bandwidth $0.2$. Note that those parameters have to be carefully set in order to capture the structure of the dataset, but this is not the current concern: we assume we already found a reasonable choice for those parameters. We add a small L2-regularization that ensures that the radius $r$ in Inequality \eqref{radeupper} is finite. The precise value of $r$ plays a secondary role as the behavior of the curve from Figure \ref{Rademachercurve} remains the same. %Details on the data generation process can be found in our code.

Looking at Figure \ref{Rademachercurve} we observe that for $\mu$ smaller than $0.1$ the curve still drops steeply, while after $0.2$ it starts to flatten out. We thus plot the resulting kernels for $\mu=0.02$ and $\mu=0.2$ in Figure \ref{circles_manifold}. We plot the isolines of the kernel around the point of class one, the red dot in the figure. We indeed observe that for $\mu=0.02$ we don't capture that much structure yet, while for $\mu=0.2$ the two concentric circles are almost completely separated by the kernel. If this procedure indeed elevates to a practical method needs further empirical testing.
\begin{figure}[t]
\subfloat[$\mu=0.02$]{\includegraphics[width=0.48\linewidth]{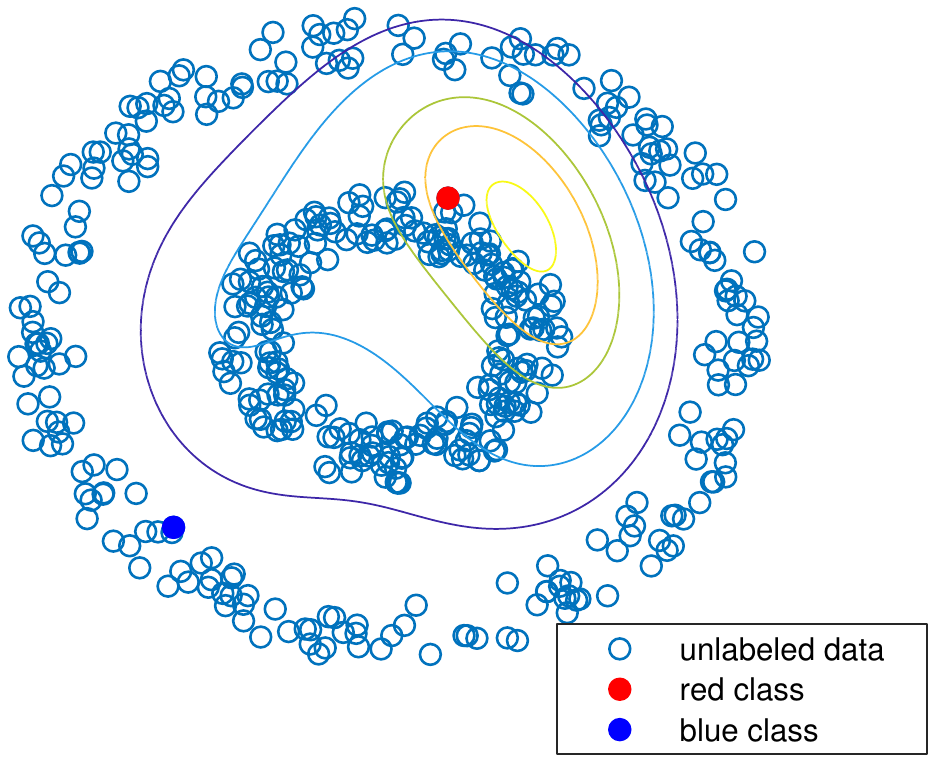}}
\subfloat[$\mu=0.2$]{\includegraphics[width=0.48\linewidth]{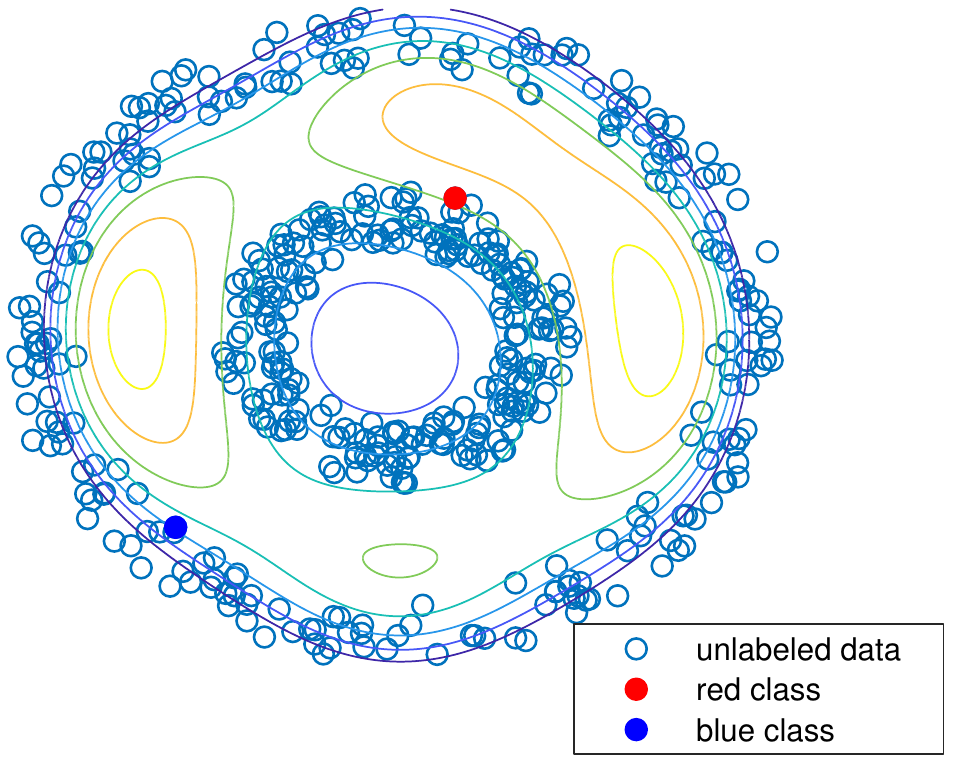} }
\caption{The resulting kernel when we use manifold regularization with parameter $\mu$ set to $0.02$ and $0.2$.}
\label{circles_manifold}
\end{figure}
\section{Discussion and Conclusion} \label{D}
This paper analysed improvements in terms of sample or Rademacher complexity for a certain class of SSL.
The performance of such methods depends both on how the approximation error of the class $\mathcal{F}$ compares to that of $\mathcal{F}^{\psi}_{\tau}$ and on the reduction of complexity by switching from the first to the latter. In our analysis we discussed the second part. The first part depends on a notion the literature often refers to as a \emph{semi-supervised assumption}. This assumption basically states that we can learn with $\mathcal{F}^{\psi}_{\tau}$ as good as with $\mathcal{F}$. Regarding our example of the two concentric circles, this would mean that each circle actually corresponds to a class. Without prior knowledge, it is unclear whether one can test efficiently if the assumption is true or not.  Or is it possible to treat just this as a model selection problem?  The only two works we know that provide some analysis in this direction are from \cite{AP}, which discusses the sample consumption to test the so-called cluster assumption, and \cite{Azizyan}, which analyzes the overhead of cross-validating the hyper-parameter coming form their proposed semi-supervised approach.

As some of our settings need restrictions, it is natural to ask whether we can extend the results. First, Lemma \ref{Lemma1} restricts us to convex optimization problems. If that assumption would be unnecessary, one may get interesting extensions. Neural networks, for example, are typically not convex in their function space and we cannot guarantee the fast learning rate from Theorem \ref{UpperNN}. But maybe there are semi-supervised methods that turn this space convex, and thus could achieve fast rates. 
In Theorem \ref{UpperNN} we have to restrict the loss to be the square loss, and \cite[Example 21.16]{Anthony} shows that for the absolute loss one cannot achieve such a result. But whether it is possible for the hinge loss, which is a typical choice in classification, is unknown to us. Corollary \ref{corlimits} considers regression and one can wonder if similar results hold for classification, e.g. when we use the hinge loss. We speculate that this is indeed true, as at least the related classification tasks, that use the $0-1$ loss, cannot achieve a rate faster than $\frac{1}{\epsilon}$ \citep[Theorem 6.8]{Shalev}.

Finally, we sketch a scenario in which sample complexity improvements of MR can be at most a constant over their supervised counterparts, ignoring logarithmic factors. This may sound like a negative result, as other methods that seem to have similar assumptions can achieve learning rates that are exponential in the number of labeled samples \citep[Chapter~6]{review}. But constant improvement can still have significant effects, if this constant can be arbitrarily large. For that consider again the example of the two concentric circles. If we set the regularization parameter $\mu$ high enough, the only possible classification functions will be the one that classifies each circle uniformly to one class, while the pseudo-dimension of the supervised model can be arbitrarily high, and thus also the constant in Corollary \ref{corlimits}. In conclusion, one should realize the significant influence constant factors in finite sample settings can have.

\bibliography{references}
\bibliographystyle{plainnat}

\end{document}